\documentclass[twoside,11pt]{article}

\usepackage{jmlr2e}
\usepackage[utf8]{inputenc} 
\usepackage[T1]{fontenc}    
\usepackage{hyperref}       
\usepackage{url}            
\usepackage{booktabs}       
\usepackage{amsfonts}       
\usepackage{amssymb}
\usepackage{amsmath}
\usepackage{nicefrac}       
\usepackage{microtype}      
\usepackage{natbib}
\usepackage{textcomp}
\usepackage{color,xcolor}
\usepackage{algorithmicx}
\usepackage{algorithm}
\usepackage{threeparttable}
\usepackage{algpseudocode}
\usepackage{subfigure}
\usepackage{placeins}
\usepackage{wrapfig}
\usepackage{graphicx} 
\usepackage{tabularx}
\usepackage{multirow,multicol, array}
\numberwithin{equation}{section}

\newcommand{\mypar}[1]{\textbf{#1.}}

\newcommand{\tabincell}[2]{\begin{tabular}{@{}#1@{}}#2\end{tabular}}  

\newcommand{\M}{\mathcal{M}}
\newcommand{\Ss}{\mathcal{S}}

\newcommand{\Pp}{\mathcal{P}}

\newcommand{\Ll}{\mathcal{L}}

\newcommand{\E}{\mathbb{E}}

\newcommand{\state}{s}
\newcommand{\st}{{\state_t}}
\newcommand{\stp}{{\state_{t+1}}}
\newcommand{\action}{a}

\newcommand{\atp}{{\action_{t+1}}}
\newcommand{\at}{{\action_t}}
\newcommand{\reward}{r}

\newcommand{\density}{p}
\newcommand{\pdyn}{\density}

\newcommand{\policy}{\pi}
\newcommand{\piparas}{\theta}
\newcommand{\qparas}{\phi}
\newcommand{\lamparas}{\xi}

\firstpageno{1}

\begin{document}

\title{Feasible Actor-Critic: Constrained Reinforcement Learning for Ensuring Statewise Safety}

\author{\name Haitong Ma$^\dag$ \email maht19@mails.tsinghua.edu.cn \\
        \name Yang Guan$^\dag$ \email guany17@mails.tsinghua.eud.cn\\
        \name Shengbo Eben Li$^{\dag, }$ \email lishbo@tsinghua.edu.cn\\
        \name Xiangteng Zhang$^\ddag$ \email zhangxt18@mails.tsinghua.edu.cn\\
        \name Sifa Zheng$^\dag$ \email zsf@tsinghua.edu.cn\\
        \name Jianyu Chen$^\parallel$ \email jianyuchen@tsinghua.edu.cn \\
       \addr School of Vehicle and Mobility$^\dag$\\
       \addr School of Aerospace and Engineering$^\ddag$\\
       \addr Institute for Interdisciplinary Information Sciences$^\parallel$,\\
       Tsinghua University\\
       Beijing, 100084, China
       }
\editor{}

\maketitle

\begin{abstract}
The safety constraints commonly used by existing reinforcement learning (RL) methods are defined only on \emph{expectation} of initial states, but allow each certain state to be unsafe, which is unsatisfying for real-world safety-critical tasks.
In this paper, we introduce the feasible actor-critic (FAC) algorithm, which is the first model-free constrained RL method that considers \emph{statewise safety}, that is safety for each initial state. 
We claim that some states are inherently unsafe \emph{no matter} what policy we choose, while for other states there exist policies ensuring safety, where we say such states and policies are \emph{feasible}. By adopting an additional neural network to first approximate the statewise Lagrange multiplier and then construct a novel statewise Lagrange function, we manage to obtain the optimal feasible policy which ensures safety for each feasible state, and the safest possible policy for infeasible states. Furthermore, the trained multiplier net can indicate whether a state is feasible or not. We provide theoretical guarantees that the constraint function and total rewards of FAC are upper and lower bounded by that of the expectation-based constrained RL methods, respectively. Experimental results on multiple tasks suggest that FAC has fewer constraint violations and higher average rewards than previous methods.
\end{abstract}


\section{Introduction}
Reinforcement learning (RL) has achieved superhuman performance in solving many complicated sequential decision-making problems like Go \citep{silver2016mastering, silver2017mastering}, Atari games \citep{mnih2013playing}, and Starcraft \citep{vinyals2019grandmaster}. Such a huge potential attracts many real-world applications like autonomous driving \citep{guan2020centralized, ma2021model}, energy system management \citep{mason2019review} and surgical robotics \citep{richter2019open}. However, most successful RL applications with these real-world applications are only with simulation platforms, for example, CARLA \citep{chen2021interpretable}, TORCS \citep{wang2018deep} and SUMO \citep{ren2020improving} in autonomous driving. The lack of safety guarantee is one of the major limitations preventing these simulation achievements from being transferred to the real world, 
which becomes an urgent problem for widely benefiting from high-level artificial intelligence.

Current safety consideration of constrained or safe RL studies are usually defined on \emph{expectations of  all possible initial states} 
\citep{altman1999constrained, RLBOOK, achiam2017constrained, ray2019benchmarking, chow2017risk, duan2019deep}.  However, there are several fatal flaws with this expectation-based safety constraint: (1) Each specific state is allowed to be unsafe as long as the expectation of states satisfies the constraint; (2) There is no way to know which states are safe and which are not, which means the safety for a specific state is almost random. Such properties are not acceptable for real-world safety-critical tasks. To avoid this problem, one might think about designing a constrained RL method to guarantee the whole state space to be safe. However, this is impossible in many realistic tasks \citep{mitchell2005time, duan2019deep, RLBOOK}. For example, Figure 1 demonstrates an emergency braking task. Intuitively, if the vehicle is too close or too fast, the collision is inevitable. In this simple case, we can actually directly calculate which states are safe and which are not shown as the green region and red region in Figure 1(i), respectively. We claim such phenomenon is common in many real-world tasks that some states are inherently unsafe \emph{no matter} what policy we choose, while for other states, there exist policies ensuring safety, where we say such states and policies are feasible. However, as shown in Figure 1(ii), the previous expectation-based method cannot approximate these regions well.

\begin{figure}[htb]
    \centering
    \includegraphics[width=0.7\linewidth]{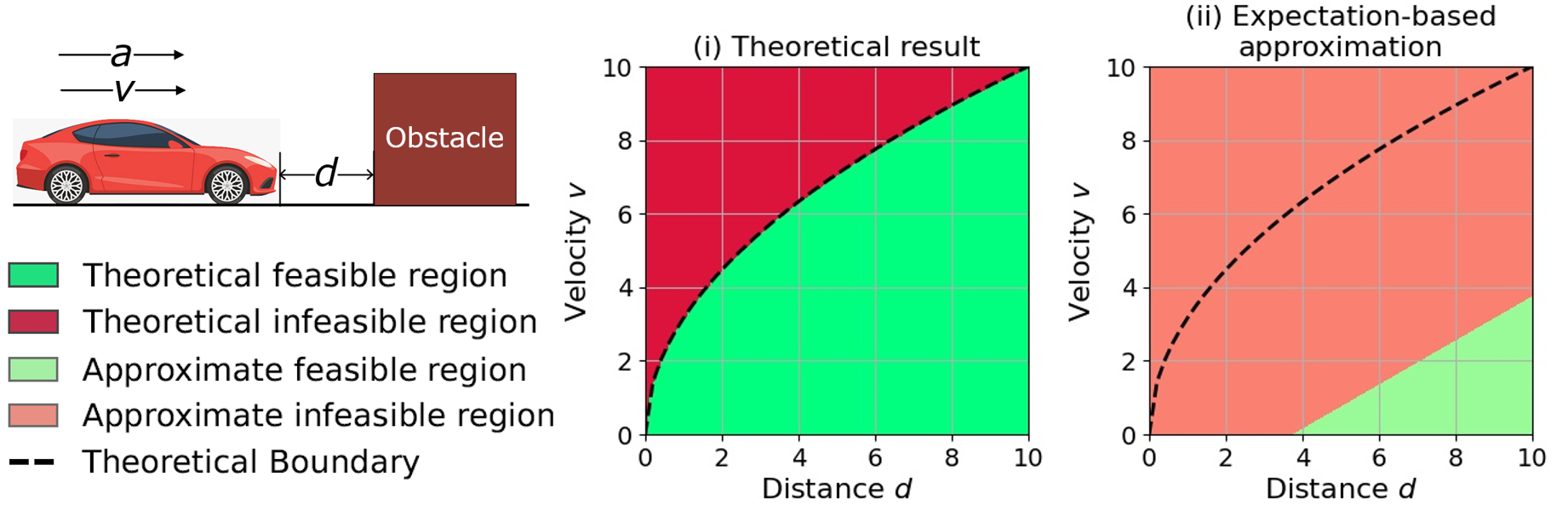}
    \caption{Theoretical feasible region and its approximation with expectation-based constraint in emergency braking. States include distance $d$ and velocity $v$, and action is the vehicle deceleration $a^\text{del}$ limited at $|a^\text{del}_{\text{max}}|=$5m/s$^\text{2}$. The real feasible region is determined by the uniform deceleration curve $d\leq v^2/2|a^\text{del}_{\text{max}}|$. The expectation-based constraint can not well describe the feasible property of states. See Appendix \ref{sec:appendixexp} for details.}
    \label{fig:region}
\end{figure}

In this paper, we propose the feasible actor-critic (FAC) algorithm to guarantee safety of all feasible initial states and indicate which states are infeasible. By adopting an additional neural network to approximate the statewise Lagrange multiplier, we can handle 
infinite number of statewise safety constraints, thus considers safety for \emph{all possible states}. The multiplier network stores additional information from the learning progress, and we can show that it can accurately indicate the feasibility of states through the statewise complementary slackness conditions and the gradient computation of multipliers. By training policy and multiplier networks with the primal-dual gradient ascent, we can obtain an optimal feasible policy that ensures safety for each feasible state, the safest possible policy for infeasible states, and a multiplier network indicating which states are infeasible. Furthermore, we provide theoretical guarantees that the constraint function and total rewards of FAC are upper and lower bounded by that of the expectation-based constrained RL methods.

Our main contributions are:\\
(1) We point out that commonly used expectation-based constraints fail to meet real-world safety requirements, and then introduce the novel definitions of statewise safety for more realistic safety criteria. We also clarify \emph{which states are possible to be safe} by defining the feasible state and region.\\
(2) We propose FAC to obtain an optimal feasible policy that ensures safety for each feasible initial state. With an additional multiplier network, FAC can indicate the feasible property of a specific state. This feasibility indication has not been obtained by any other existing RL method.\\
(3) Experimental results on multiple tasks and safety constraints suggest that FAC has fewer constraint violations and higher average rewards than previous methods. The feasibility indicating ability is also verified.
\section{Preliminaries}
\subsection{Constrained Markov Decision Process}
The constrained Markov decision process (CMDP), $\M=\langle \mathcal{S}, \mathcal{A}, \Pp,r, c\rangle$, is to maximize the expected return of rewards while satisfying the constraint on the expected return of costs:\\
\begin{equation}
	\max_{\pi}\ J\big(\pi\big) =  \E_{\tau\sim\pi}\Big\{\sum_{t=0}^{\infty}\gamma^t r_t\Big\} \quad \text{s.t.} \ C\big(\pi\big) =  \E_{\tau\sim\pi}\Big\{\sum_{t=0}^{\infty}\gamma_c^tc_t\Big\}\leq d
\label{eq:cmdp} \tag{EP}
\end{equation}
where $\mathcal{S}$ and $\mathcal{A}$ are state and action space, $\pi$ is the policy in policy set $\Pi$, $\Pp$ is the environment transition model, $r,\ c:\mathcal{S} \times \mathcal{A}  \times \mathcal{S}\rightarrow \mathbb{R}$ is the reward and cost, $\gamma,\ \gamma_c$ are their discounted factors, $\tau=\{\state_0,\action_0,\state_1,\action_1\dots\}$ is the trajectory under policy $\pi$ starting from an initial state distribution $d_0(\state)$ \citep{sutton2018reinforcement,altman1999constrained}. The feasible policy set is defined as $\Pi_C=\big\{\pi|C\big(\pi\big)\leq d \big\}$, but CMDP does not discuss which states are feasible. Notably, the percentile indicators rather than mean value are used to define worst-case safety constraints in some studies \citep{chow2017risk, yang2021wcsac}. Nevertheless, they are still wrapped by the trajectory expectation, so the problem that certain state is allowed to be unsafe still exists. Our motivation is to deal with the failure of safety guarantee caused by computing the expectation on initial state distribution, which is different from these worst-case studies.
\subsection{Related Works}
\mypar{Constrained RL with expectation-based safety constraint} We focus on recent practical constrained RL algorithms for continuous state space. One branch of recent studies is to compute a constraint-satisfying policy gradient, for example, using Rosen projection \citep{uchibe2007constrained} or the local approximate second-order solver similar to natural policy gradient (PG) \citep{achiam2017constrained, yang2020projection, zhang2020first}. These local solvers have the recovery issue, which means that the recovery mechanism must be manually designed in case that there exists no constraint-satisfying policy gradient at a specific step. Furthermore, the approximation error and recovery mechanism makes these constrained RL methods hard to implement and also harms their performance. Another branch is the Lagrangian-based constrained RL, which is to 
optimize the linear combination of expected returns and costs with a scalar multiplier. Lagrangian-based modification is applied with vanilla PG, actor-critic (AC), trust-region policy optimization (TRPO), and proximal policy optimization (PPO) \citep{chow2017risk, schulman2015trust, schulman2017proximal, ray2019benchmarking}. Lagrangian-based approach is simple to implement and outperforms the constrained gradient methods in some complex tasks \citep{ray2019benchmarking, stooke2020responsive}. However, the oscillation issue seriously affects the performance of Lagrangian approach. \citet{stooke2020responsive, peng2021separated} use the over-integral in the control theory to explain this phenomenon, and propose the PID-Lagrangian method to handle the oscillation problem.\\
\mypar{Shielding RL or RL with controllers} 
Some studies in the control theory have similar motivations of statewise safety, like safety barrier certificate or reachability theory \citep{mitchell2005time, ames2019control, ma2021model}. With these control-based approaches to compute a safe action set at a specific state, some RL studies project the action into the safe set to achieve safety of the current state \citep{cheng2019end, dalal2018safe, pham2018optlayer}. These methods is simple to implement, but have difficulties in handling stochastic policy and guaranteeing the policy optimality. The major limitation of shields or controllers is that they require a prior model and can only predict at a specific state, while FAC is a model-free constrained RL algorithm with a feasibility indicator of the whole state space. We do not consider these methods in the baselines since prior models are not available for the empirical environments. 


\section{Statewise Safety and Feasibility Analysis}
\label{sec:feas}
In this section, we introduce the complete definition about statewise safety constraint, and the feasible partition of state space. How to interpreter the feasibility of a given state by the Lagrange multiplier is also explained.
\subsection{Definitions of Statewise Safety}
\begin{definition}[Feasible state under $\pi$] Given a policy $\pi$, we define that a state is feasible, or safe if its expected return of cost $c_t$ satisfies the inequality:
\begin{equation*}
    v^\pi_C(\state) =  \mathbb{E}_{\tau\sim\pi}\Big\{\sum\nolimits_{t=0}^{\infty}\gamma_c^tc_t\bigg|\state_0=\state\Big\}\leq d
\end{equation*}
where $\mathbb{E}_{\tau\sim\pi}\{\cdot|\state_0=\state\}$ denotes that the expected value of trajectory generated by policy $\pi$ starting from a given initial state $s$. $v^\pi_C(s)$ is named as safety critic, or cost value function. If the constraint is violated under policy $\pi$ at state $s$, the state $s$ is defined to be unsafe.
\end{definition}

Before defining safety constraints, we first define the feasible partition of state space:
\begin{definition}[Infeasible region]
States those are unsafe no matter what policy we choose are defined as the \textbf{infeasible state}. The infeasible region $\mathcal{S}_I$ is defined as the set of all infeasible states:
\begin{equation*}
    \mathcal{S}_I = \{s|v^\pi_C(s)>d,\ \forall\pi \in \Pi\}
\end{equation*}
\end{definition}
\begin{definition}[Feasible region]
The feasible region $\mathcal{S}_F$ is defined as the complementary set of the infeasible region
\begin{equation*}
    \mathcal{S}_F = \complement_\mathcal{S}\mathcal{S}_I
\end{equation*}
Note that the feasible and infeasible regions are irrelevant with policy $\pi$. However, the policy may be not good enough to guarantee that all possibly feasible states $s\in\mathcal{S}_F$ to be feasible, so we define feasible region \textbf{under policy }$\mathbf{\pi}$:
\begin{equation*}
    \mathcal{S}^\pi_F = \big\{s\ \big|\ v_C^\pi(s)\leq d \big\}
\end{equation*}
\end{definition}
It is obvious that $\mathcal{S}^\pi_F \subseteq \mathcal{S}_F$. Denote $\mathcal{I}$ as the set of all possible initial states. Ideally, if $\mathcal{I} \subseteq\mathcal{S}_F$, we can guarantee all initial states to be feasible, or safe. However, in practice, $\mathcal{I}$ may include infeasible states, and we can not find an appropriate policy for these infeasible states. To define a meaningful feasible property of policy, we consider the intersection of $\mathcal{I}$ and $\mathcal{S}_F$:
\begin{definition}[Statewise safety constraint]Define the feasible initial state set as $\mathcal{I}_F=\mathcal{I}\cap \mathcal{S}_F$. The statewise safety constraint is defined to guarantee every state in the feasible initial state set to be safe:
\begin{equation}
    v^\pi_C(\state) \leq d, \forall \state \in \mathcal{I}_F \label{eq:cstr}
\end{equation}
We define a policy to be feasible if it satisfies the statewise safety constraint. For those infeasible initial states $s\in\complement_\mathcal{I}\mathcal{I}_F$, consistently optimizing them will damage the performance. We can actually identify them by our method in the following.
\label{df:statewisecstr}
\end{definition}
The feasible policy set under statewise safety constraint is $\Pi_F=\big\{\pi\big|v^\pi_C(s)\leq d, \forall s \in \mathcal{I}_F\big\}$. Equivalently, $\Pi_F=\big\{\pi\big|\mathcal{I}_F\subseteq\mathcal{S}^\pi_F\big\}$. Finally, without changing the RL maximization objective, we formulate the optimization problem with statewise safety constraint:
\begin{equation}
	\max_{\pi}\ J\big(\pi\big) =  \E_{\tau\sim\pi}\Big\{\sum\nolimits_{t=0}^{\infty}\gamma^t r_t\Big\}
    \quad	\text{s.t.} \   v^\pi_C(\state)\leq d, \forall s \in \mathcal{I}_F
\label{eq:statewiseop} \tag{SP}
\end{equation}

\subsection{Multipliers: Statewise Feasibility Indicators}
\label{sec:cs}
We use the Lagrangian-based approach to solve problem \eqref{eq:statewiseop}. As each state in $\mathcal{I}_F$ has a constraint, there exists the corresponding Lagrange multiplier at each state, and we denote the statewise multiplier as $\lambda(\state)$. The resulting Lagrange function is named as \emph{original statewise Lagrangian}, which is shown as:
\begin{equation}
	\Ll_{\text{ori-stw}}\big(\pi,\lambda\big) = -\E_{\state\sim d_0(\state)}v^\pi(\state) + \sum\nolimits_{\state \in \mathcal{I}_F} \lambda(\state)\Big(v_C^{\pi}\big(\state\big)-d\Big)
	\label{eq:SL1}
	\tag{O-SL}
\end{equation}
The multipliers have the physical meaning of solution feasibility. We explain the relation between multipliers and feasibility first by the statewise complementary slackness condition:
\begin{proposition}[Statewise complementary slackness condition]
    \label{prop:scsc}
	For the problem \eqref{eq:statewiseop}, at state $\state$ the optimal multiplier and optimal safety critic are $\lambda^*(\state), v^{*}_C(\state)$, the following conditions hold:\footnote{We assume the strong duality holds here for simplicity, and the following experimental results depict that the feasibility indicator still works in complex environments where the strong duality assumption may fail.}
	\begin{equation}
		\lambda^*(\state)=0\ , v^{\pi*}_C(\state)< d, \text{ or } \lambda^*(\state) > 0,\  v^{\pi*}_C(\state) = d
	\end{equation}
\end{proposition}
The proposition comes from the Karush-Kuhn-Tucker (KKT) necessary condition for the problem \eqref{eq:statewiseop}. If $v^{\pi}_C(\state) = d$, we say that the constraint at state $\state$ is active. An active constraint represents that the constraint is preventing the objective function to be further optimized, otherwise the state steps into the infeasible region. A state whose safety constraint is not active must lie inside the feasible region.
In this way, the optimal multiplier can indicate whether the constraint of a feasible state $s\in\mathcal{I}_F$ is active or not. However, in practical, it is not helpful for the identification of those infeasible states $s\in\complement_\mathcal{I}\mathcal{I}_F$. Thanks to the following Corollary, we can further distinguish the infeasible states by the multipliers, pick out them and find the approximate optimal feasible solution. 
\begin{corollary}
    \label{cor:1}
	If $\state$ lies in infeasible region, then $\lambda(\state) \to \infty$ with the primal-dual ascent.
\end{corollary}
The main idea is that if the feasible solution do not exists, the gradient of $\lambda$ will always be positive for state $s$. Detailed proof is provided in Appendix \ref{sec:proofcol}. 
As suggested by the Proposition \ref{prop:scsc} and Corollary \ref{cor:1}, the approximate relation between multipliers and feasibility can be established, as shown in Table \ref{table:region}.
\begin{table}[htb]
	\centering
	\begin{threeparttable}[h]
		\begin{tabular}{cc}
			\toprule
			Multiplier scale $\lambda(\state)$ & Feasibility situation of $\state$\\
			\midrule
			Zero & Inactive (inside feasible region)\\
			Finite & Active (on the boundary of feasible region) \\
			Infinite$\dagger$ & Infeasible region\\
			\bottomrule
		\end{tabular}
	\begin{tablenotes}
     \item[$\dagger$] A heuristic threshold is set to indicate infinity practically.
    \end{tablenotes}
	\end{threeparttable}
	\caption{Approximate relation between multipliers and feasibility.}
	\label{table:region}
	\end{table}
\section{Feasible Actor-Critic}
\label{sec:algo}
In this section, we provide details about the feasible actor-critic algorithm. Firstly, we introduce the difficulties of handling the original Lagrangian \eqref{eq:SL1}, which is also the reason why this type of safety constraint is not considered by existing studies. Then we provides theoretical performance analysis and implementation instructions. 
\subsection{Statewise Lagrangian for Practical Implementation}
For practical implementation with RL, 
the sum term $\sum_s\lambda(\state)(v_C^{\pi}(\state)-d)$ in \eqref{eq:SL1} is intractable since we have no access to the feasible region $\mathcal{I}_F$ but only a state distribution $d_0$. Besides, the summation over infinite state set is impractical in continuous domain. This is the reason why the existing constrained RL methods do not consider this type of safety constraint. We try to adapt the original statewise Lagrangian \eqref{eq:SL1} to fit the paradigm of sample-based learning. 
For this, we first formulate the \emph{statewise Lagrangian} as:
\begin{equation}
		\mathcal{L}_\text{\rm stw}(\pi,\lambda) = \E_{\state\sim d_0(\state)}\big\{-v^{\pi}(\state) + \lambda(\state)\big(v_C^{\pi}(\state)-d\big)\big\}
		\label{eq:asl}
		\tag{SL} 
\end{equation}
Then, we propose a theorem for the equivalence of \eqref{eq:SL1} and \eqref{eq:asl}.
\begin{theorem}[Equivalence of \eqref{eq:SL1} and \eqref{eq:asl}]
  If the optimal policy and Lagrange multiplier mapping $\pi^*$ and $\lambda^*$ exist for problem $\max_{\lambda}\inf_{\pi}\mathcal{L}_\text{\rm stw}(\pi,\lambda)$, then $\pi^*$ is the optimal policy of problem \eqref{eq:statewiseop}.
	\label{theorem:1}
\end{theorem}

The main idea of this proof is to scale the safety constraints by the probabilistic density, that is, $d_0(\state)\big(v^\pi_C(\state)-d\big)\leq 0$.
As $d_0(\state)\geq 0$ and for $\state\notin\mathcal{I}_F$, $d_0(s)=0$, the scaled constraint is equivalent to the statewise safety constraint \eqref{eq:cstr}. For those infeasible initial states $s\in\complement_\mathcal{I}\mathcal{I}_F$, the multiplier still goes to infinity. See Appendix \ref{section:ap1} for detailed proof. 

\subsection{Performance Comparison with Expected Lagrangian Methods}
\label{sec:compare}
In this section, we provide theoretical analysis on performance of FAC. Two theorems are proposed to discuss the upper bound of constraint satisfaction and the lower bound of total rewards, which are exactly those with the expectation-based safety constraint, respectively.
\begin{theorem}
	If all possible initial states are feasible or $\mathcal{I}\subseteq\mathcal{S}_F$, a feasible policy $\pi_f$ under statewise constraints must be feasible under expectation-based constraints, that is, for the same constraint threshold $d$, $\Pi_F\subseteq\Pi_C$.
	\label{theorem:cstr}
\end{theorem}
See Appendix \ref{sec:proof1} for its proof. The theorem indicates that a feasible policy under the criterion of expectation-based constraints may become "unsafe" in certain states, suggesting the necessities of the statewise constraints in safety-critical tasks. 

The performance comparison is under the concave-convex assumption both on policy and state space. The Lagrange function of the optimization problem with expectation constraint is:
\begin{equation}
    \Ll_{\text{exp}}(\pi,\lambda)\doteq\lambda \big(C\big(\pi\big)-d\big)-J\big(\pi\big) = \lambda\Big(\E_{\tau\sim\pi}\Big\{\sum\nolimits_{t=0}^{\infty}\gamma_c^tc_t\Big\}-d\Big)-\E_{\tau\sim\pi}\Big\{\sum\nolimits_{t=0}^{\infty}\gamma^tr_t\Big\}
    \label{eq:expectedlag} \tag{EL}
\end{equation}
We name the Lagrange function \eqref{eq:expectedlag} as the \emph{expected Lagrangian}. 
The optimal Lagrangian and total rewards of \eqref{eq:expectedlag} and \eqref{eq:asl} are denoted by $\mathcal{L}^*_{\text{\rm exp}}, J^*_{\text{\rm exp}}$ and $\mathcal{L}^*_{\text{\rm stw}},J^*_{\text{\rm stw}}$.
\begin{theorem}
\label{theorem:per}
	Assume that $\mathcal{I}_F\ \text{and}\ \Pi$ are both nonempty convex set. $v^\pi$ is concave on $\mathcal{I}_F$ and $\Pi$, and $v^\pi_C$ is convex on $\mathcal{I}_F$ and $\Pi$. The optimal expected Lagrangian \eqref{eq:expectedlag} is the upper bound of the optimal statewise Lagrangian \eqref{eq:asl}, that is, $\Ll^*_{\text{\rm stw}}\leq \Ll^*_{\text{\rm exp}}$. If the Slater's condition holds on $\Pi$ for each $\state$, then $\Ll^*_\diamond=-J^*_\diamond$, $\diamond\in${\rm \{stw, exp\}}. We further obtain the total rewards lower bound of statewise Lagrangian \eqref{eq:asl} as 
	\begin{equation}
	    J^*_{\text{\rm stw}}\geq J^*_{\text{\rm exp}}
	\end{equation}
\end{theorem}
See Appendix \ref{sec:proof2} for its proof. The main idea of proof is that the statewise multipliers find the optimal solution of the dual problem on each feasible initial state, while the expected Lagrangian optimizes the dual problem with only a \emph{single} dual variable or multiplier. Assumptions made in this theorem may be a little strong. Nevertheless, the thought of optimizing the dual problems in a statewise manner does introduce more flexibility to obtain better  Lagrangian, and further better policies practically.
\subsection{Practical Algorithm}
\label{sec:pracalgo}
\label{sec:code}
	\begin{algorithm}[htb]
	\caption{Feasible Actor-Critic}
	\label{alg:factotal}
	\begin{algorithmic}
		\Require $\qparas_1$, $\qparas_2$, $\qparas_C$, $\piparas$,  $\xi$.\Comment{Initial parameters}
		\State $\bar\diamondsuit$ $\leftarrow$ $\diamondsuit$ for $\diamondsuit\in\{$$\qparas_1$, $\qparas_2$, $\qparas_C$, $\piparas\}$ \Comment{Initialize target network weights}
		\State $\mathcal{B}\leftarrow\emptyset$ \Comment{Initialize an empty replay buffer}
		\For{each iteration}
		\For{each environment step}
		\State $\at \sim \pi_\piparas(\at|\st), \stp \sim \pdyn(\stp| \st, \at)$ 
		\Comment{Sample transitions}
		\State $\mathcal{B} \leftarrow \mathcal{B} \cup \left\{(\st, \at, \reward(\st, \at), c(\st, \at), \stp)\right\}$
		\Comment{Store the transition in the replay buffer}
		\EndFor
		\For{each gradient step}
		\State $\qparas_i \leftarrow \qparas - \beta_Q \hat \nabla_{\qparas_i} J_Q(\qparas_i)$ for $i \ \in \{1,2,C\}$\Comment{Update the Q-function weights}
		\If{gradient steps \texttt{mod} $m_\pi$ $=0$}
		\State $\piparas \leftarrow \piparas - \beta_\policy \hat \nabla_\piparas J_\policy(\piparas)$\Comment{Update policy weights}
		\State $\alpha \leftarrow \alpha - \beta_\alpha \hat \nabla_\alpha J_{\alpha}(\alpha)$ \Comment{Adjust temperature}
		\EndIf
		\If{gradient steps \texttt{mod} $m_\lambda$ $=0$}
		\State $\lamparas \leftarrow \lamparas + \beta_\lambda \hat \nabla_\lamparas J_\lambda(\lamparas)$\Comment{Update multipliers weights}
		\EndIf
		\State $\bar\diamondsuit$ $\leftarrow$ $\tau\diamondsuit+(1-\tau)\bar\diamondsuit$ for $\diamondsuit\in\{\qparas_1$, $\qparas_2$, $\qparas_C$, $\piparas\}$ \Comment{Update target network weights}
		\EndFor
		\EndFor
		\Ensure $\qparas_1$, $\qparas_2$, $\qparas_C$, $\piparas$,  $\xi$.
	\end{algorithmic}
\end{algorithm}

We optimize the statewise Lagrangian \eqref{eq:asl} based on the off-policy maximum entropy RL framework. A notable difference is that we incorporate a neural network to approximate Lagrange multiplier $\lambda$ with parameters $\xi$, that is, $\lambda(s)\approx \lambda_\xi(s)$. We need to train the value function $Q_{\qparas}$, cost value function $Q_{\qparas_C}$, policy $\pi_\theta$, and the multipliers $\lambda_\xi$ listed in Table \ref{tab:func}.
\begin{table}[h]
	\centering
	\begin{tabular}{cccc}
		\hline
		Function name & Function type & Notions & Weights \\
		\hline
		Value function & soft Q-function & $Q_\qparas(\st,\at)$ & $\qparas$ \\
		Cost value function & Q-function & $Q_{\qparas_C}(\st,\at)$ & $\qparas_C$ \\
		Policy & stochastic policy with $tanh$ bijector & $\pi_\theta(\st)$ & $\piparas$ \\
		Multipliers & multiplier function & $\lambda_\xi(\st)$ & $\xi$ \\
		\hline
	\end{tabular}
	\caption{Parameterized Function}
	\label{tab:func}
\end{table}

We use the primal-dual gradient ascent to alternatively update policy and multiplier networks. The pseudocode is shown in Algorithm 1. $J_Q(\qparas_i), J_\pi(\theta), J_\alpha(\alpha), J_\lambda(\xi)$ represent the losses of values (two values and one cost value), policy, temperature and multiplier updates. $\beta$ is their learning rates. 
The loss and gradient of soft Q-function, or the reward value function $Q_\qparas$ is exactly the same as those in the study by \citet{haarnoja2018soft}. The update of cost Q-function is:
\begin{equation*}
J_{Q}(\qparas_C)=\mathbb{E}_{\left(\st, \at\right) \sim \mathcal{B}}\left\{\frac{1}{2}\left(Q_{\qparas_C}\left(\st, \at\right)-\left(c\left(\st, \at\right)+\gamma_C \mathbb{E}_{\stp\sim p, \atp \sim \pi}\left\{Q_{\bar{\qparas_C}}\left(\stp, \atp\right)\right\}\right)\right)^{2}\right\}
\end{equation*}
The stochastic gradient to optimize cost Q-function is 
\begin{equation*}
\hat{\nabla}_{\theta} J_{Q}(\qparas_C)=\nabla_{\qparas_C} Q_{\qparas_C}\left(\st, \at\right)\bigg(Q_{\qparas_C}\left(\st, \at\right)-\Big(c(\st, \at)+\gamma_C Q_{\bar{\qparas_C}}\left(\st, \at\right)\Big)\bigg)
\end{equation*}
According to Theorem \ref{theorem:1}, we expand the $v^\pi(s), v^\pi_C(s)$ with the soft Q-function and Q-function, respectively, to formulate the policy loss:
\begin{equation}
J_{\pi}(\piparas)=\mathbb{E}_{\st \sim \mathcal{B}}\bigg\{\mathbb{E}_{\at \sim \pi_{\piparas}}\Big\{\alpha \log \big(\pi_{\piparas}\left(\at \mid \st\right)\big)-Q_{\qparas}(\st, \at) + \lambda_\xi(\st)\big( Q_{\qparas_C}(\st,\at)-d\big)\Big\}\bigg\}
\label{eq: losspi}
\end{equation}
The policy gradient with the reparameterized policy $\at=f_\piparas(\epsilon_t; \st)$ can be approximated by:
\begin{equation*}
\begin{aligned}
\hat{\nabla}_{\piparas} J_{\pi}(\piparas)=&\nabla_{\piparas} \alpha \log \big(\pi_{\piparas}\left(\at \mid \st\right)\big)+\\
&\Big(\nabla_{\at} \alpha \log \left(\pi_{\piparas}\left(\at \mid \st\right)\right)-\nabla_{\at} \big(Q_\qparas\left(\st, \at\right) - \lambda_\xi(\st) Q_{\qparas_C}\left(\st, \at\right)\big)\Big) \nabla_{\piparas} f_{\piparas}\left(\epsilon_{t} ; \st\right)
\end{aligned}
\end{equation*}
where the threshold $d$ is neglected since it is irrelevant with $\theta$.
The loss function for updating multiplier net is the same as \eqref{eq: losspi}, which is simplified by
\begin{equation*}
J_{\lambda}(\xi) = \mathbb{E}_{\st \sim \mathcal{B}}\bigg\{\mathbb{E}_{\at \sim \pi_{\piparas}}\Big\{+ \lambda_\xi(\st)\big( Q_{\qparas_C}(\st,\at)-d\big)\Big\}\bigg\}
\end{equation*}
where the entropy and reward value function term are neglected since they are irrelevant with the multipliers. The stochastic gradient is
\begin{equation}
\hat{\nabla}J_\lambda(\lamparas) = \big( Q_{\qparas_C}(\st,\at)-d\big) \nabla_\lamparas\lambda_\lamparas(\st)
\label{eq:lamsubgrad}
\end{equation}

Additionally, we introduce some tricks used in practical implementations. The primal-dual gradient ascent has poor convergence performance practically. Inspired by exiting studies about learning with adversarial objective function and delayed policy update tricks \citep{goodfellow2014generative, fujimoto2018addressing}, we set a different interval schedule for policy delay steps $m_\pi$ and ascent delay steps $m_\lambda$ to improve stability. The oscillation issue is also an inherent problem for primal-dual ascent. Rather than adopting the PID controller to enforce stability \citep[for example,][]{peng2021separated,stooke2020responsive}, We start training the multiplier net until the cost value function gets close to the constraint threshold. These tricks significantly improve the stability of the training process.
\section{Empirical Analysis}
We choose two different sets of experiments to evaluate the safety enhancement of FAC under different constraint formulations: (1) MuJuCo robots walking with speed limits; (2) Safety Gym agents exploring safely while avoiding multiple obstacles. The speed constraint is a rather easy one, while the constraint of safe exploration is more realistic and challenging. All environments are implemented with Gym API and MuJoCo physical simulator \citep{brockman2016openai, ray2019benchmarking, todorov2012mujoco}. 

We compare our algorithm against commonly used constrained RL baseline algorithms \citep{achiam2017constrained, ray2019benchmarking, stooke2020responsive}, including Constrained Policy Optimization (CPO), TRPO-Lagrangian (TRPO-L) and PPO-Lagrangian (PPO-L). They are all with the expectation-based constraint. TRPO-L and PPO-L use the primal-dual gradient ascent to compute the saddle point of expected Lagrangian \eqref{eq:expectedlag} with a scalar multiplier, and the objective function $J(\pi)$ is the surrogate formulations in TRPO or PPO. The metrics include both constraint satisfaction and total rewards, and a good policy is to get as high returns as possible while satisfying constraints. Details about baselines and the hyperparameters are provided in Appendix \ref{sec:implementations}.

\subsection{Controlling Robots with Speed Limit}
\begin{figure}[htb]
	\centering
	\subfigure{\includegraphics[width=0.3\linewidth]{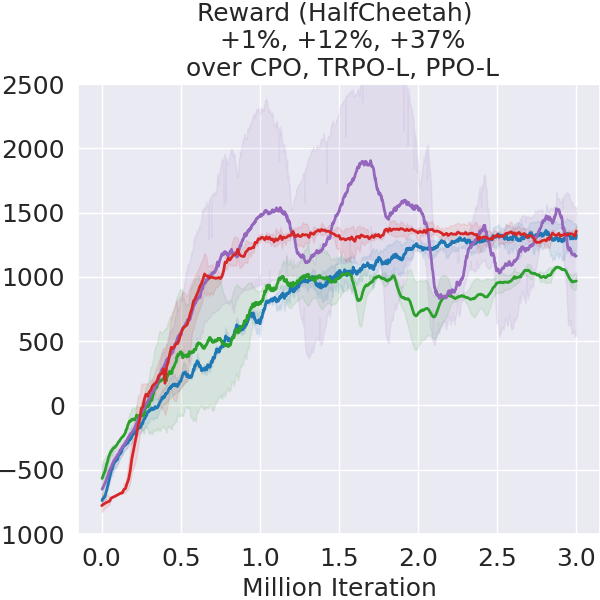}}
	\subfigure{\includegraphics[width=0.3\linewidth]{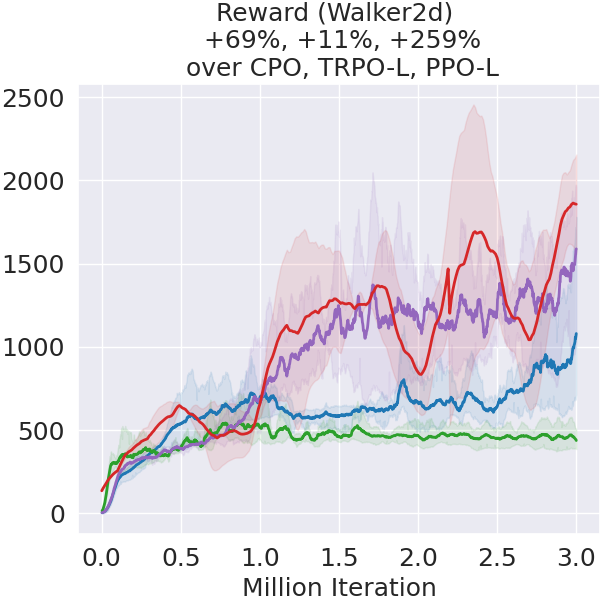}}
	\subfigure{\includegraphics[width=0.3\linewidth]{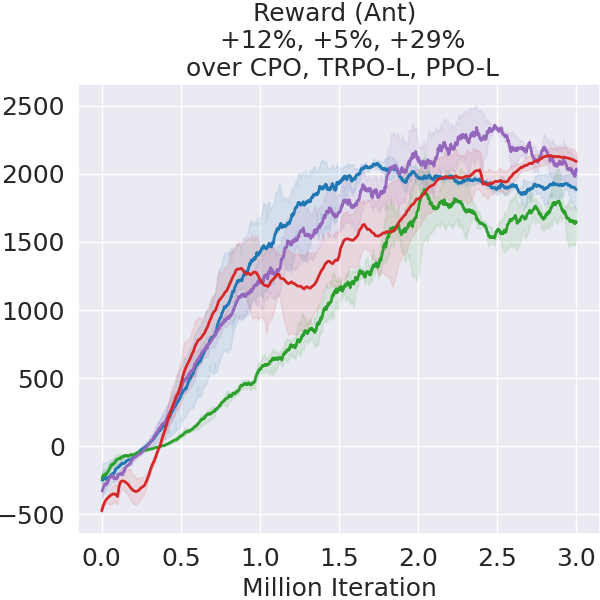}}
	\subfigure{\includegraphics[width=0.3\linewidth]{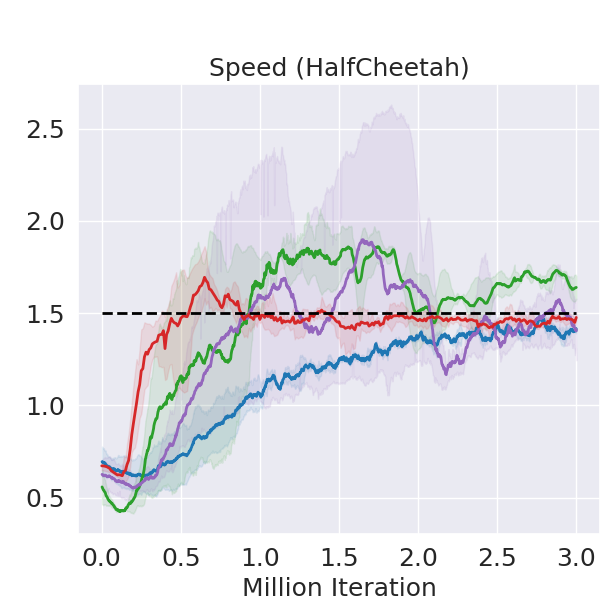}}
	\subfigure{\includegraphics[width=0.3\linewidth]{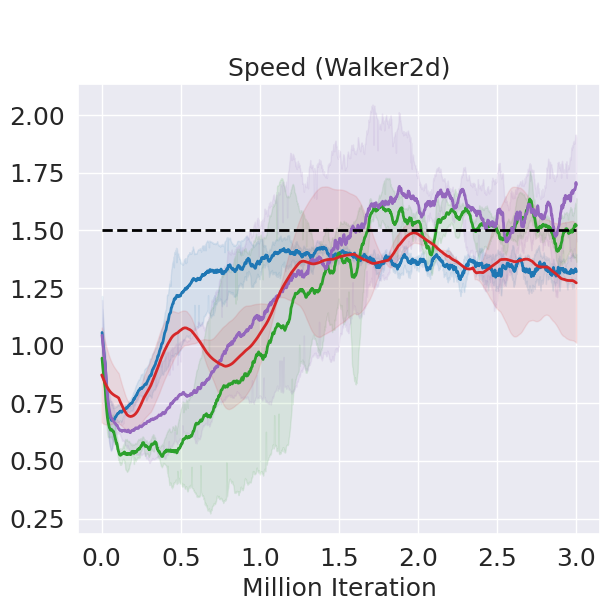}}
	\subfigure{\includegraphics[width=0.3\linewidth]{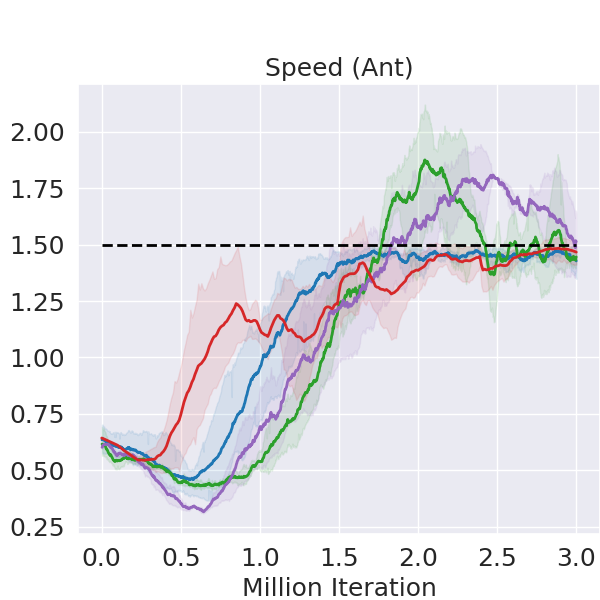}}
	\subfigure{\includegraphics[width=0.6\linewidth]{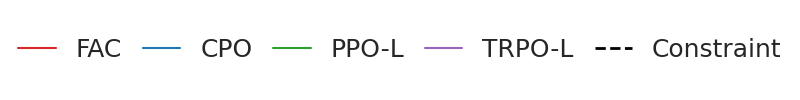}}
	\caption{Learning curves for controlling robots with speed limit tasks. The x-axis represents the training iterations and the y-axis represents the average total rewards and robot speed of the last 20 episodes. The solid lines represent the mean value over 5 random seeds. The shaded regions represent the 95\% confidence interval. FAC consistently enforces constraint satisfaction with highest average rewards.}
	\label{fig:gym}
\end{figure}

We choose three MuJoCo environments where we attempt to train a robotic agent to walk with speed limit. Figure \ref{fig:gym} shows that FAC outperforms other baselines in terms of reward on all three tasks while enforcing the speed constraint. As for the expectation-based Lagrangian methods, TRPO-L and PPO-L outperform CPO in some environments, which is consistent with the results of \citet{ray2019benchmarking, stooke2020responsive}. However, they suffer from the oscillation issues, and the constraints are violated in some environments while FAC consistently enforces the constraint satisfaction of not only the mean value, but also most of the \emph{error bar}. FAC is robust in all environments, while baseline methods might perform not well in some environments. Especially in the rather challenging task of Walker2d, PPO-L fails to learn a meaningful policy. 
\begin{table}[htb]
	\centering
	\begin{tabular}{cccccc}
			\toprule
			Environment & \tabincell{c}{Performance\\(speed limit:1.5)} & FAC  & CPO & TRPO-L & PPO-L \\\midrule
			\multirow{2}{*}{HalfCheetah-v3} &Return &\textbf{1318}$\pm$\textbf{26.3}  & 1303$\pm$120 & 967$\pm$55.9& 1180$\pm$432  \\
			&Speed &1.46$\pm$0.013 & 1.41$\pm$0.011 & 1.43$\pm$0.125& 1.63$\pm$0.065   \\
			\multirow{2}{*}{Walker2d-v3} &Return &\textbf{1651}$\pm$\textbf{314}  & 979$\pm$400 & 1483$\pm$384 & 460$\pm$68.8 \\
			&Speed &1.30$\pm$0.230 & 1.32$\pm$0.051 & 1.67$\pm$0.197  & 1.50$\pm$0.088 \\
			\multirow{2}{*}{Ant-v3}& Return& \textbf{2121}$\pm$\textbf{68.4} & 1898$\pm$111&2017$\pm$119& 1646$\pm$117\\
			& Speed& 1.48$\pm$0.052 & 1.44$\pm$ 0.053&1.52$\pm$0.150& 1.43$\pm$0.032  \\
			\bottomrule 
		\end{tabular}
		\caption{Average return and speed of the last iteration over 5 random seeds on controlling robots with speed limit environments. The bold ones are the best performance among algorithms. $\pm$ corresponds to a single standard deviation over runs.}
	\label{table:epcost}
\end{table}

\subsection{Safe Exploration Tasks}
We choose two Safety Gym environments with different agents and tasks, called Point-Button (PB) and Car-Goal (CG). In both environments, we control the agent to collect as many as bonuses distributed among a 2D arena while avoiding collisions with static or moving obstacles as best as we can. Especially, the cost signal is encoded by a binary variable, which provides $c_t=1$ once any unsafe action is taken, otherwise $c_t=0$. The safety constraints are the dangerous action rate $c_\text{rate}=\sum c_i/\text{(sampled steps)}\leq 10\%$. It is transformed to the safety critic constraint $v_C^\pi(\state)\leq 10$. Details of the transformation is provided in Appendix \ref{sec:valuecstr}. Besides, we call an episode is dangerous when the dangerous action rate constraint is violated. We count the cost value distribution, the episodic dangerous action rate and the number of dangerous episodes as the metrics to evaluate the constraint satisfaction performance.
\begin{figure}[htb]
	\centering
	\subfigure[FAC]{\includegraphics[width=0.32\linewidth]{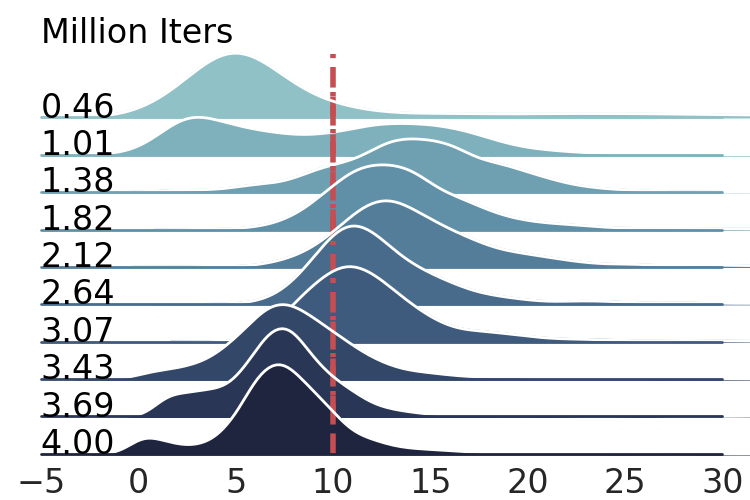}}
	\subfigure[TRPO-Lagrangian]{\includegraphics[width=0.32\linewidth]{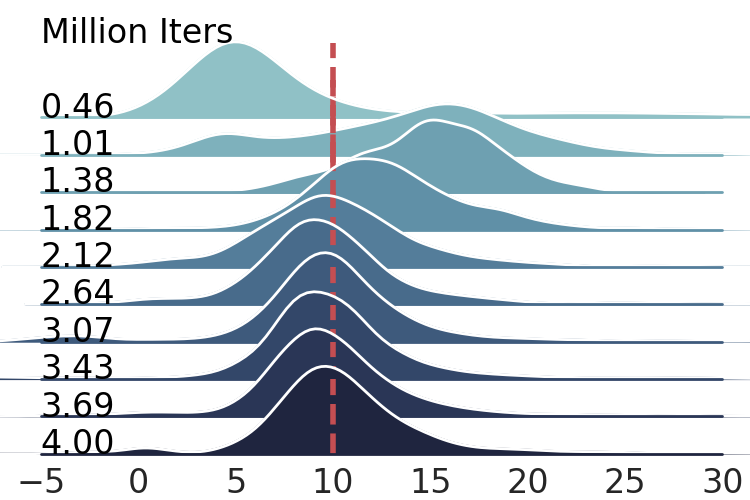}}
	\subfigure[CPO]{\includegraphics[width=0.32\linewidth]{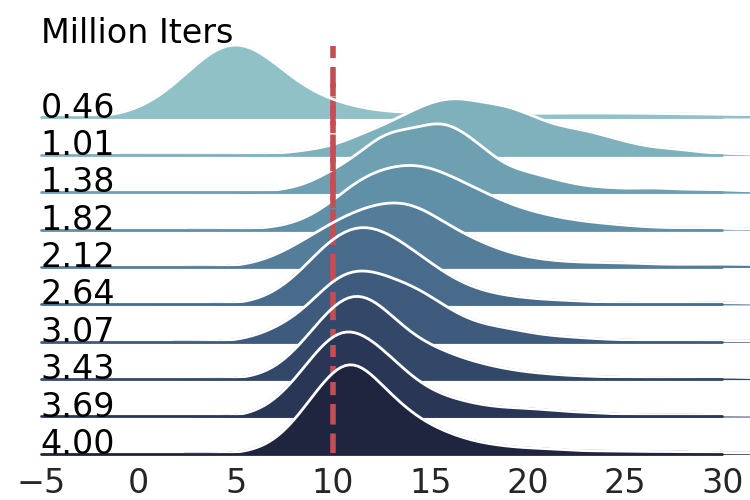}}
	\caption{Distribution of batch cost value $v^\pi_C(s)$ in Point-Button safe exploration tasks during training. The histogram on one Axis represents the cost value distribution of batch states on the specific training iteration labeled at left. The red dash line is the constraint threshold $10$. The results of TRPO-L and PPO-L are similar and we choose TRPO-L as an example for clearer demonstration.}
	\label{fig:batchvc}
\end{figure}

To better understand the constraint satisfaction property with statewise safety constraints, we visualize the cost value distribution over a batch of state during training in Figure \ref{fig:batchvc}. Results show that most states in the distribution with FAC satisfy the safety constraint. Inevitably, there exist some infeasible states in the batch in practical implementation, so not all states satisfy the safety constraint after training. With TRPO-L, the peak of distribution lies on the constraint threshold approximately. The distribution is approximately symmetric, so the fact that the peak lies on the constraint threshold corresponds to that the expectation satisfies the safety constraint. Obviously with the expectation constraint satisfied, half of the batch states are still violating the safety constraint. The constraint-violating states are even more for CPO since CPO does not obey the expectation-based constraint.  
\begin{figure}[h]
	\centering
    \subfigure{\includegraphics[width=0.3\linewidth]{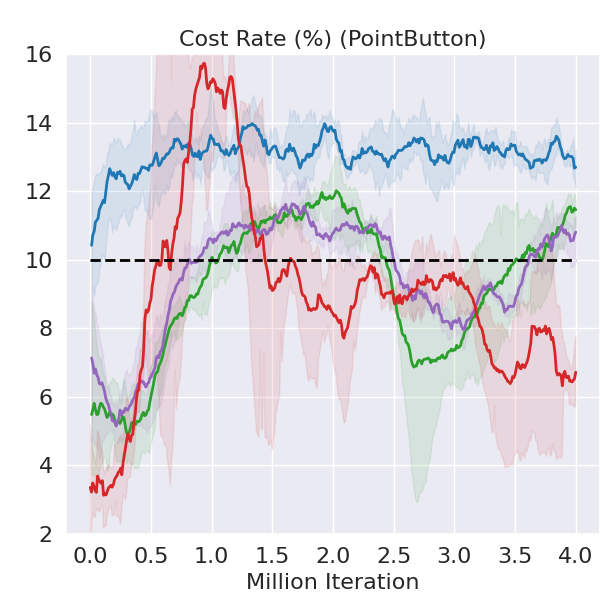}}
	\subfigure{\includegraphics[width=0.3\linewidth]{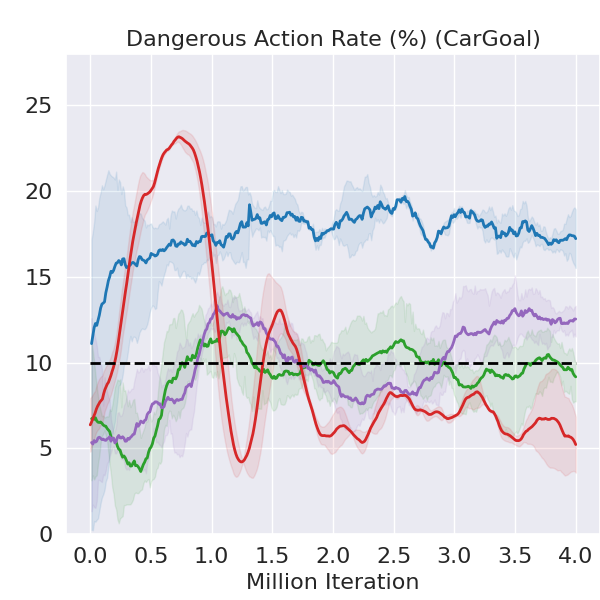}}\\
	\subfigure{\includegraphics[width=0.3\linewidth]{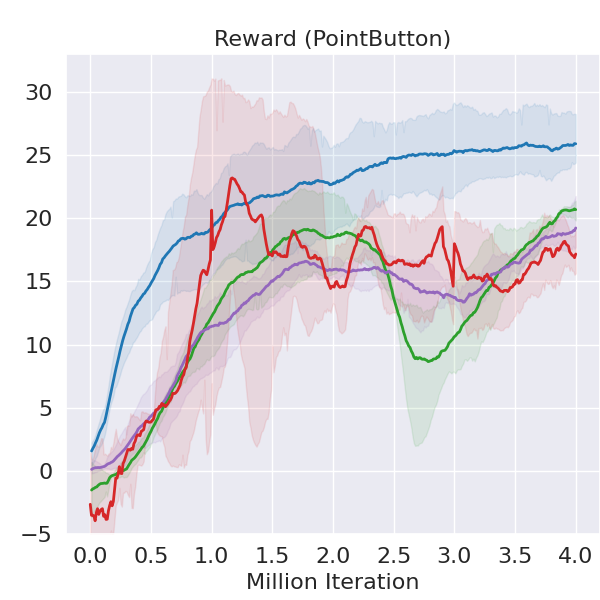}}
	\subfigure{\includegraphics[width=0.3\linewidth]{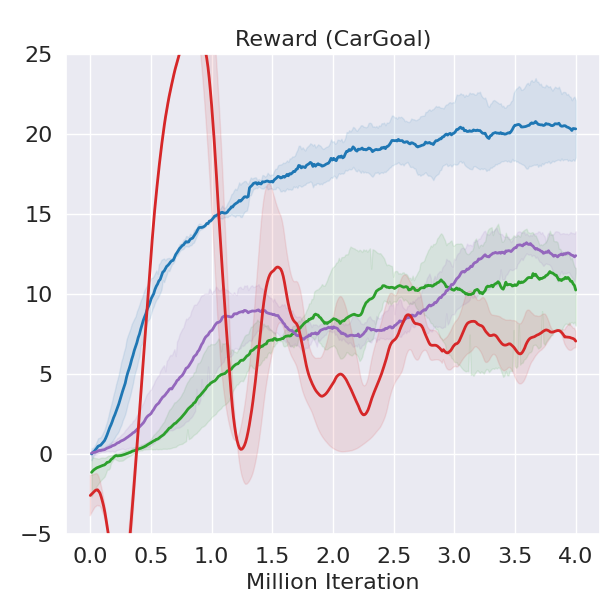}}
	\subfigure{\includegraphics[width=0.6\linewidth]{figure/results/legend.png}}
	\caption{\small{Learning curves for safe exploration tasks. Axis means the same as Figure \ref{fig:gym}. FAC also enforces the error bar to satisfy the safety constraint, while baseline algorithms barely confines the average value to satisfy the dangerous action rate safety constraint.}}
    \label{fig:safexp}
\end{figure}
\begin{table}[h]
	
	\centering
		\begin{tabular}{cccccc}
	\toprule
	Env & \tabincell{c}{Performance\\(Dangerous action\\ rate constraint: 10\%)}& FAC & CPO & TRPO-L & PPO-L \\\midrule
	\multirow{4}{*}{PB}         & Return& 17.08$\pm$3.55  &  25.67$\pm$1.96  &18.76$\pm$1.42     &     20.48$\pm$0.63           \\
	&$c_\text{rate}$ (\%) & \textbf{6.372}$\pm$\textbf{1.764}   &13.19$\pm$0.688   &10.47$\pm$5.357   &    11.08$\pm$0.855           \\
	&\tabincell{c}{Dangerous episodes\\in 100 tests} & \textbf{3}   &73  &52  &  66         \\\midrule
	\multirow{4}{*}{CG}           &Return &7.51$\pm$0.56   &18.62$\pm$3.25      &12.49$\pm$1.31 &10.99$\pm$2.25           \\
	&$c_\text{rate}$ (\%)&\textbf{6.086}$\pm$\textbf{2.508}             &16.71$\pm$1.874  &12.29$\pm$1.636       &9.815$\pm$1.064            \\
	&\tabincell{c}{Dangerous episodes\\in 100 tests}& \textbf{8}   &77  &69   &  47\\
	\bottomrule
\end{tabular}
\caption{Average return and dangerous action rate over 5 random seeds, and number of dangerous episodes in 100 test episodes on safe exploration environments. The Return and $c_{\text{rate}}$ are the performance of the last training iteration. And the 100 tests are carried out using the final trained networks. The bold ones are the minimum dangerous action rates or the minimum number of dangerous episodes among different algorithms. $\pm$ corresponds to a single standard deviation over runs. FAC achieves remarkable lowest dangerous episodes in both environments with reasonable performance sacrifice.}
	\label{table:safeexp}
\end{table}

Table \ref{table:safeexp} shows the numerical summaries of safe exploration tasks, and training curves can be found in Figure \ref{fig:safexp}. Results demonstrate that FAC not only confines the average episodic dangerous action rate to satisfy the constraint, but \emph{most of the error range} to satisfy the constraint, which means almost all test episodes are enforced to be safe. The remarkable lowest numbers of dangerous test episodes verify the effectiveness of our method. As for baseline algorithms, CPO fails to satisfy the expectation-based constraints in all safe exploration tasks, which is consistent with observation made by \citet{ray2019benchmarking}. For TRPO-L or PPO-L, although they yield a constraint-satisfying result for their expectation-based constraints, there are still half of the dangerous episodes. It suggests that expectation-based safety constraints are not enough to guarantee the safety of a trajectory, while FAC improves this in a large margin. The trade-off between the reward and cost returns in safe exploration environments is consistent with results of \citet{ray2019benchmarking}, so the performance sacrifice of FAC is reasonable.

\subsection{Feasibility Indicator by Multiplier Network}

\begin{figure}[H]
	\centering
	\subfigure[Infeasible region in PB: turning and waiting for moving blocks.]{\includegraphics[width=0.4\linewidth]{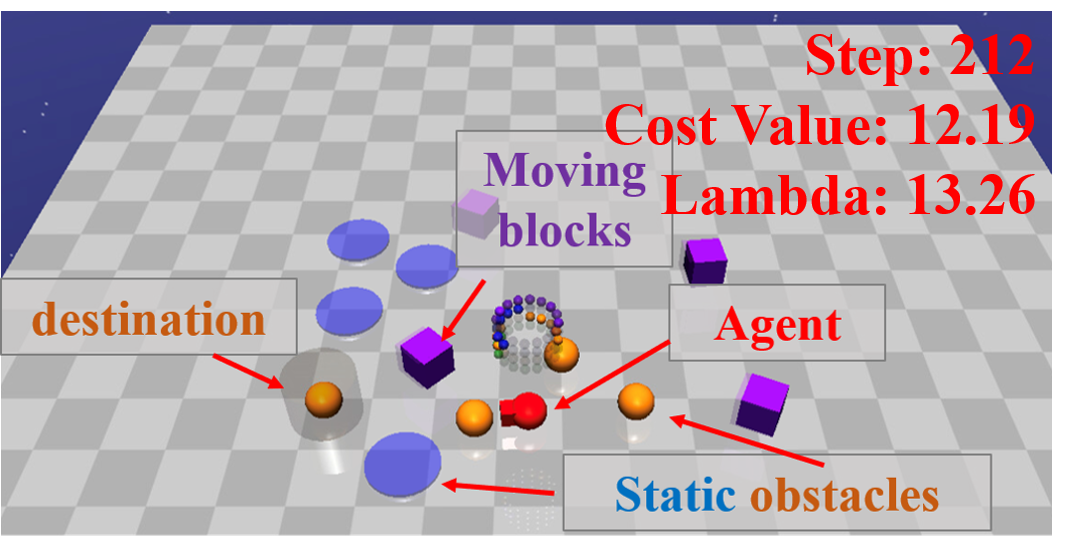}}
	\subfigure[Boundary of feasible region in PB: bypassing blocks.]{\includegraphics[width=0.4\linewidth]{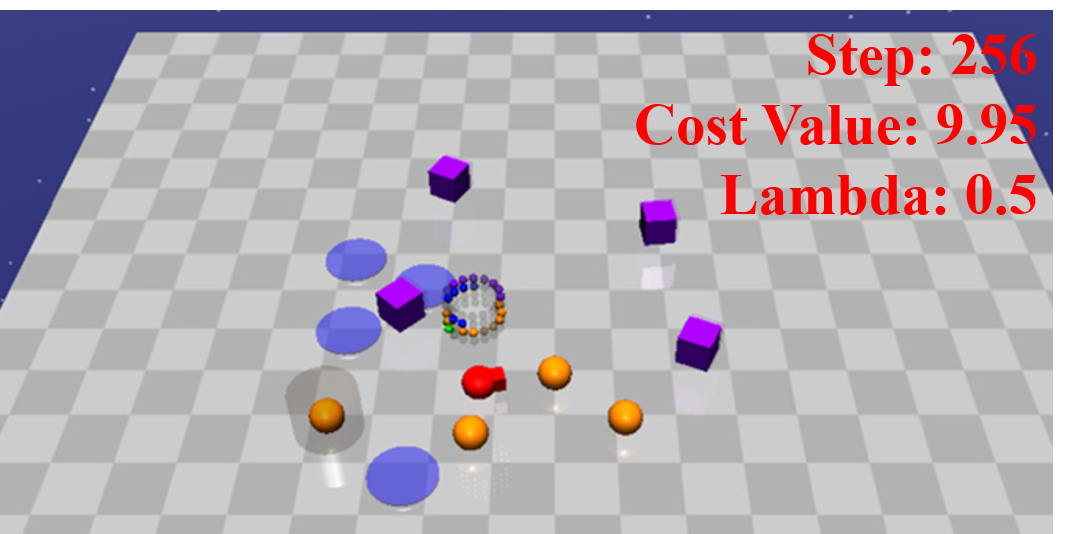}}
	\subfigure[Inside feasible region in PB: going straight.]{\includegraphics[width=0.4\linewidth]{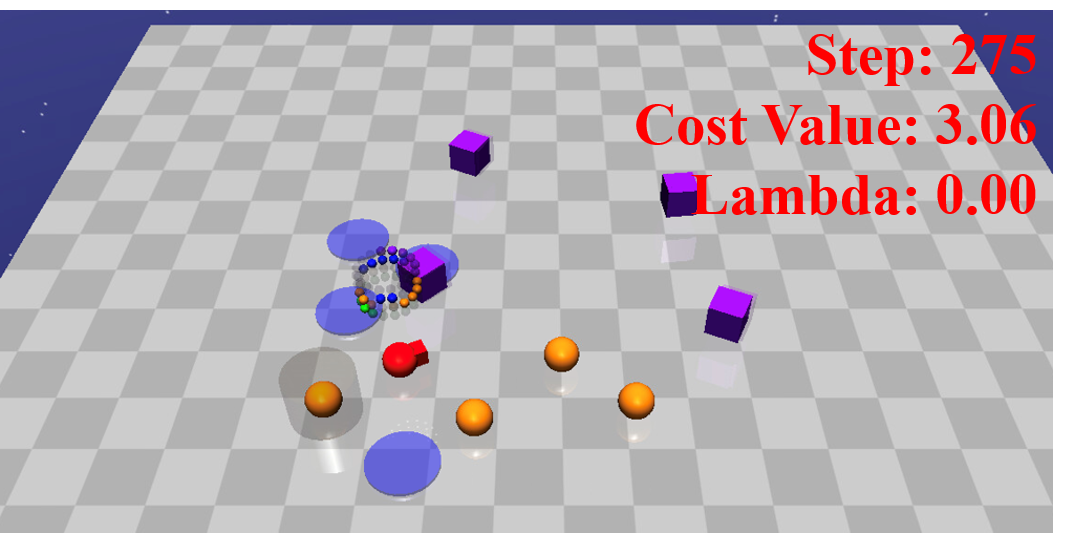}}\\
	\subfigure[Boundary of feasible region in CG: bypassing blocks.]{\includegraphics[width=0.4\linewidth]{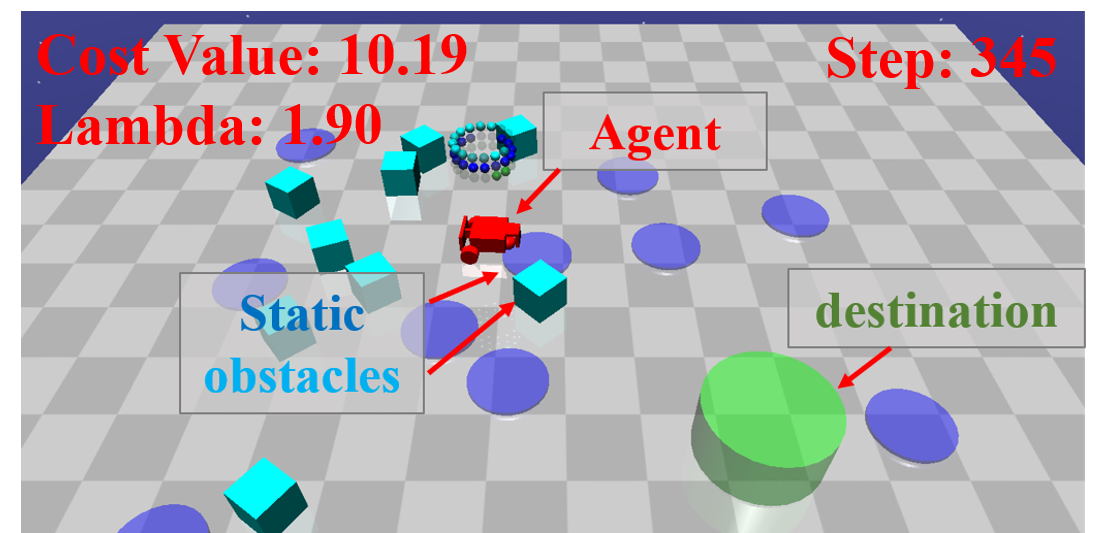}}
	\subfigure[Inside feasible region in CG: going straight.]{\includegraphics[width=0.4\linewidth]{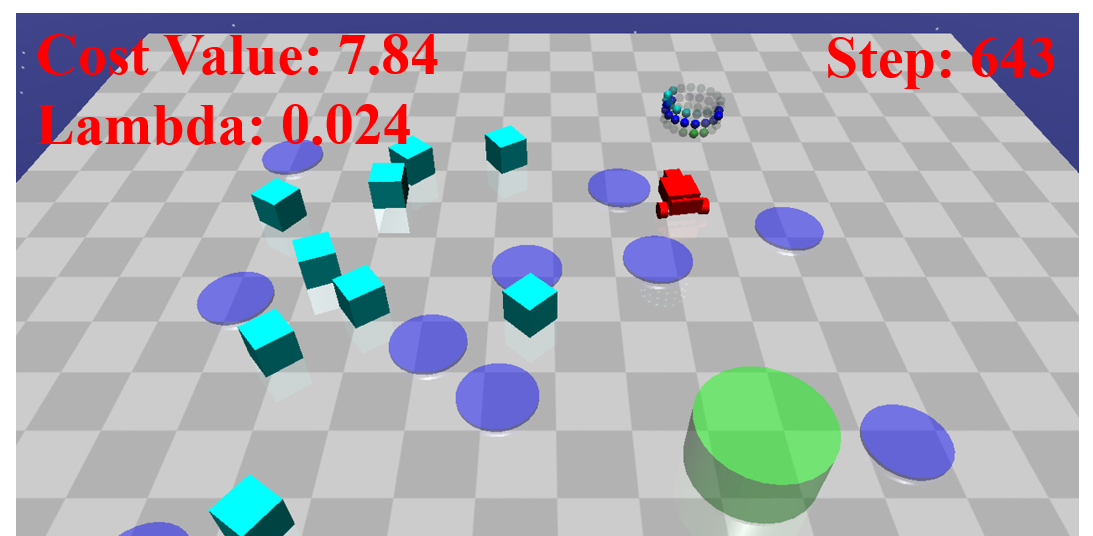}}	
	\subfigure[Inside feasible region in CG: going straight.]{\includegraphics[width=0.4\linewidth]{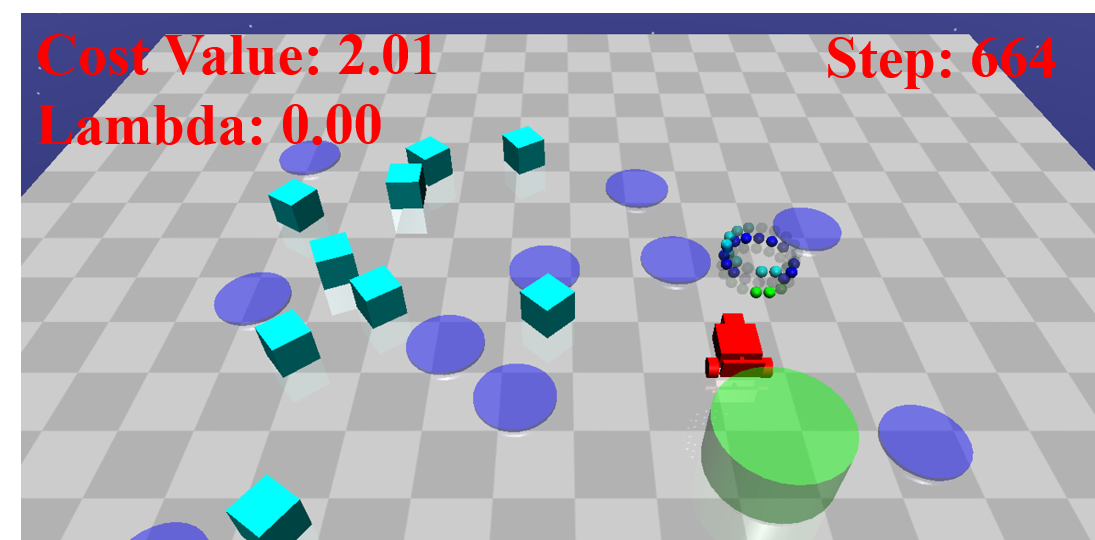}}		
	\caption{Feasibility indication of Point-Button (a-c) and Car-Goal (d-f) tasks with by FAC in sequential representative frames. The time steps, safety critic values (shorted as cost value) and multiplier output are listed inside the figures. The captions show indication results and actions taken by agents. Details of first episode can be found in \url{https://youtu.be/WQxXtqoBINg}.}
	\label{fig:cs}
\end{figure}
We select one episode from each safe exploration task to demonstrate the ability to indicate feasibility with the multiplier network. One of them demonstrates how the agent restores from the infeasible region, and the other demonstrates how the agent keeps itself away from the infeasible region by staying on the boundary. Figure \ref{fig:cs} shows three representative frames in each episode. The heuristic threshold of infinity is $2$, which is determined by the scale of value function and cost value function. Details about determining this threshold are analyzed in Appendix \ref{sec:boundary}.

In Figure \ref{fig:cs}(a), the multiplier exceeds the infinite threshold and the multiplier is rather large, which indicates that the agent lies in the infeasible region. The agent chooses to turn around and wait for the moving blocks (the purple cube) to go away, which is the safest policy for standing still. Thanks to the moving block giving way, the multiplier decreases quickly, suggesting that state in (b) changes to a feasible state lying on the boundary of feasible region. The agent bypasses the static obstacles since going straight will leave the feasible region. Finally, when the agent approaches the destination in (c), the multiplier becomes zero. Obviously, the agent lies in the feasible region without active safety constraints. The similar progress with (b)-(c) happens in (d)-(f). The agent chooses to bypass the static obstacles by firstly heading to the upper zone where the obstacles are sparse, then turning around to the destination.
Notably, the cost value function can also indicate the activeness of the constraints. However, this judgement requires highly accurate value estimation, which is hard for RL value estimations. As shown in Figure \ref{fig:cs} (b) and (d), the constraints are active with finite multipliers, but cost values are not strictly equal to the threshold. Furthermore, if constraint violation happens like Figure \ref{fig:cs}(a) and (d), only the cost value function can not distinguish whether the state lies in infeasible region. These cases are critical since we must avoid infeasible region like (a), but can improve policy to make (d) feasible.
\section{Concluding Remarks}
We introduced the concept of statewise safety, which requires that the safety constraints starting from arbitrary feasible initial states should be satisfied. We proposed the feasible actor-critic (FAC) algorithm to guarantee statewise safety with a multiplier network as the feasibility indicator. FAC is easy to implement and shows effective guarantee of statewise safety in several simulated robotics tasks. The multiplier network can provide accurate feasible region indication. We also provided theoretical guarantees that the constraint function and total rewards of FAC are upper and lower bounded by that of the expectation-based constrained RL methods. Safety issues are supremely critical to applying RL in real-world tasks. We believe that a meaningful safety threshold should be the statewise formulation starting from each initial state, rather than from the initial distributions. Therefore, the statewise safety constraint must be taken into consideration, and we suggest safe RL researchers to focus on the statewise safety. Moreover, the multiplier net actually learns and stores the information not explored by any other previous RL algorithms. Further exploiting the multiplier network may inspire new paradigms when considering safety in the RL society. FAC still suffers from inherent problems of primal-dual ascent, like oscillations and sensitivity to the update of the multiplier network, which will be addressed in the future works. A more accurate judgement from the multiplier to the feasibility is also considered. There are some promising directions related to the feasibility information provided by the multiplier net. For example, it can be used to design safe exploration rules.


\acks{This study is supported by National Key R\&D Program of China with 2020YFB1600200. This study is also supported by Tsinghua University Toyota Joint Research Center for AI Technology of Automated Vehicle.}


\newpage

\appendix
\section{Proof of Theorems}
\label{sec:proof}
\subsection{Proof of Corollary \ref{cor:1}}
\label{sec:proofcol}
There always not exists feasible solution means that at state $\state$, the inequality always holds:
\begin{equation*}
    v^\pi_C(\state) > d
\end{equation*}
According to the gradient of multiplier network \eqref{eq:lamsubgrad}, $\hat{\nabla}J_\lambda$ is always greater than 0, which drives the multiplier $\lambda(s)$ to infinity.
\subsection{Proof of Theorem \ref{theorem:1}}
\label{section:ap1}
Considering only the feasible region $\Ss_F$ since we can not find a policy to enforce infeasible states to be safe, and the alternative Lagrangian in equation (\ref{eq:asl}) with a parameterized policy $\pi(\state;\theta)$ is:
\begin{equation*}
    \begin{aligned}
    J_\Ll(\theta,\lambda) & =\E_{\state\sim d^{0}(\state)}\bigg\{v^{\pi}\big(\state) + \lambda(\state)\Big[v_C^{\pi}\big(\state\big)-d\Big]\bigg\}\\
    & = \E_{s\sim d^{0}(s)}\big\{v^{\pi}\big(s)\big\} + \sum_{s\in\Ss_F}d^{0}(s)\lambda(s)\Big[v^{\pi}_C(s)-d\Big]\\
    & = \E_{s\sim d^{0}(s)}\big\{v^{\pi}\big(s)\big\} + \sum_{s\in\Ss_F}\lambda(s)\bigg\{d^{0}(s)\Big[v^{\pi}_C(s)-d\Big]\bigg\}\\
    \end{aligned}
\end{equation*}
Construct an alternative constrained optimization problem here, whose formulation is exactly corresponding to the optimization problem \eqref{eq:statewiseop}:
\begin{equation}
    \begin{aligned}
        \min_{\theta}&\ \E_{s\sim d^{0}(s)}v^{\pi}(s)\\
        \text{s.t.}&\ d^{0}(s)\Big[v^{\pi}_C(s)-d\Big]\leq 0, \forall s \in \Ss_F
    \end{aligned}
    \tag{A-SP}\label{eq:asp}
\end{equation}
We focus on the constraint formulation of problem (\ref{eq:asp}). The visiting probability has the property of:
\begin{equation*}
    d^{0}(s)\geq 0 
\end{equation*}
The initial feasible state set has the property
\begin{equation*}
    \mathcal{I}_F = \Big\{s \big| d^{0}(s)>0\Big\}
\end{equation*}
According to the definition of possible initial state set, for those $s$ with $d^{0}(s)=0$, $s\notin\mathcal{\mathcal{I}}$. For those $s$ with $d^{0}(s)=0$, that is, $s\in\mathcal{\mathcal{I}}$ the constraints in problem (\ref{eq:asp}) is equivalent to $\Big[v^{\pi}_C(s)-d\Big]\leq 0$. Therefore, the constraint can be reformulated to 
\begin{equation*}
    v^{\pi}_C(s)-d\leq 0, \forall s \in \mathcal{\mathcal{I}_F}
\end{equation*}
which is exactly the definition of feasible policy.
$\hfill\blacksquare$

\subsection{Proof of Theorem \ref{theorem:cstr}}
\label{sec:proof1}
The expectation-based constraints can be reformulated as
\begin{equation*}
    \begin{aligned}
         C\big(\pi\big) =& \E_{\tau\sim\pi}\bigg\{\sum_{t=0}^{\infty}\gamma_c^tc_t\bigg\} \\
         =& \sum_{s}d_0(\state)\E_{\tau\sim\pi}\bigg\{\sum_{t=0}^{\infty}\gamma_c^tc_t\big|\state_0=\state\bigg\}\\
         =& \E_{\state\sim d_0(\state)}\big\{v^\pi_C(\state)\big\}
    \end{aligned}
\end{equation*} 
According to the Definition \ref{df:statewisecstr}, for a policy $\pi_f\in\Pi_F$ and $\forall \state \in \mathcal{I}\subseteq\mathcal{S}^{\pi_f}_F$, $v^{\pi_f}_C(\state)\leq d\  $. Therefore, the expectation on the initial state distribution 
\begin{equation*}
	C(\pi_f)=\E_{\state\sim d_0(\state)}\big\{v^{\pi_f}_C(\state)\big\} = \sum_s d_0(\state)v^{\pi_f}_C(\state) \leq d
\end{equation*}
$\hfill\blacksquare$
\subsection{Proof of Theorem \ref{theorem:per}}
\label{sec:proof2}
We regard the state as a random variable in optimizing problem (\ref{eq:statewiseop}). Then for each state, we can construct a Lagrange function as 
\begin{equation}
	L\big(\pi, \lambda, s\big) = - v^{\pi}\big(s) + \lambda(s)\Big[v_C^{\pi}\big(s\big)-d\Big]
	\label{eq:stwlag}
\end{equation}
The Lagrange dual problem is
\begin{equation}
	\max_{\lambda}\inf_{\pi}L(\pi,\lambda ,s) = \max_{\lambda}\inf_{\pi}\bigg\{-v^\pi(s)+\lambda [v^\pi_C(s)-d]\bigg\}
\end{equation} 
We denote $G(\lambda, s)$ as the dual problem:
\begin{equation}
	G(\lambda ,s)=\inf_{\pi}\bigg\{-v^\pi(s)+\lambda v^\pi_C(s) - \lambda d \bigg\}
\end{equation}
The expected Lagrangian optimize the expected solution of $G$.
\begin{lemma}[Convex condition for infimum operation \citep{bertsekas1997nonlinear}]
	If $C$ is a convex nonempty set, the function $f$ is convex in $(x, y)$, then the infimum on $y$
	\begin{equation}
		g(x)=\inf_{y\in C}f(x,y)
	\end{equation}
	is concave.
	\label{lemma:infconc}
\end{lemma}
\begin{proposition}
	The dual problem of :
	\begin{equation}
		G(\lambda ,s)=\inf_{\pi}\bigg\{-v^\pi(s)+\lambda v^\pi_C(s) - \lambda d\bigg\}
		\label{eq:dp}
	\end{equation}
	is concave on $\Ss$.
\end{proposition}
\begin{proof}
	According to the concave-convex assumption and the linear of convexity, for each $\lambda$,
	\begin{equation}
		-v^\pi\big(s\big)+\lambda v^\pi_C\big(s\big) -\lambda d
	\end{equation} is convex on $(s, \pi)$. Therefore, according to Lemma \ref{lemma:infconc}, for each $\lambda$, $G(\lambda ,s)$ is concave on $\Ss$.
\end{proof}

\begin{lemma}[Lower bound on deterministic equivalent \citep{shapiro2014lectures}]
	For a stochastic programming problem with the optimization variable $x$ and random variable $\omega$, if $f$ is convex in $\omega$ for each $x$, then
	\begin{equation}
		f(x, \E \omega) \leq \E f(x, \omega)
	\end{equation}
	\label{lemma: ineq}
\end{lemma}

We use $G^*_{\rm exp}$ to denote the optimal solution of dual problems of (\ref{eq:expectedlag}). Then we get
\begin{equation}
    \begin{aligned}
        G^*_{\text{\rm exp}}=&\max_\lambda \bigg\{\inf_\pi\bigg\{ -\E_{s}v^{\pi}(s)+\lambda \E_{s}v^{\pi}_C(s)\bigg\}\bigg\}\\
        \geq&\max_\lambda \bigg\{\inf_\pi\bigg\{ -v^\pi(\mathbb{E}_s s)+\lambda v^\pi_C(\mathbb{E}_s s)\bigg\} \bigg\}= \max_\lambda G(\lambda,\E_{s} s)
    \end{aligned}
    \label{eq:ineq1}
\end{equation}
The inequality holds because the Jensen inequality under the concave-convex assumption.
\begin{lemma}[Infinite fitting power of policy and multipliers \citep{RLBOOK}]
    If the policy $\pi(\cdot)$ has infinite fitting power, then
    \begin{equation*}
        \inf_\pi\mathbb{E}_s\{\cdot\}=\mathbb{E}_s\{\inf_\pi(\cdot)\}
    \end{equation*}
\end{lemma}According to the definition of the statewise Lagrangian \eqref{eq:asl}, we get
\begin{equation}
\begin{aligned}
    G^*_{\text{\rm stw}}&=\max_\lambda \inf_\pi \mathbb{E}_s \bigg\{-v^\pi(s) + \lambda(s) [v^{\pi}_C (s)-d]\bigg\}\\
    &=\max_\lambda \E_{s} \bigg\{\inf_{\pi}\big\{-v^\pi(s)+\lambda(s) v^\pi_C(s) - \lambda(s) d \big\}\bigg\}\\
    &= \max_\lambda \E_{s} \big\{G(\lambda(s), s)\big\}\\
    &\leq \E_{s} \max_\lambda \big\{G(\lambda(s), s)\big\}
\end{aligned}
\label{eq:ineq2}
\end{equation}
The equality between swapping the expectation and infimum is because if the policy $\pi$ has infinite fitting ability. The inequality about expectation and max is analyzed by \citet{thrun1993issues}.\footnote{If the multiplier net also has infinite fitting power here, the inequality is replaced by the equality. We do not need the infinite fitting ability of multiplier approximation here.}
\begin{proposition}
	The optimal solution of dual problem \eqref{eq:dp} $\max_\lambda G(\lambda,s)$ is concave on $\Ss$.
\end{proposition}
\begin{proof}
	According to the property of dual problem, the domain of dual problem (\ref{eq:dp}) is convex, and dual problem is concave on $\lambda$ \citep{bertsekas1997nonlinear}. As we already give the convavity of $G(\lambda,s)$ on $s$, we can use Lemma \ref{lemma:infconc} again and get the concavity of $\max_\lambda G(\lambda,s)$.
\end{proof}
According to \textbf{Lemma \ref{lemma: ineq}} and combining \eqref{eq:ineq1} and \eqref{eq:ineq2}, we get:
\begin{equation*}
	\begin{aligned}
	    G^*_{\text{\rm exp}}&\geq\max_\lambda G(\lambda,\E_{s} s)\\
	    &\geq\E_{s} \max_{\lambda} \big[G(\lambda, s)\big]\\
	    &\geq\E_{s} \max_{\lambda(s)} \big[G(\lambda(s), s)\big]\\
	    &\geq G^*_{\rm stw}
	\end{aligned}
\end{equation*}
Then we get:
\begin{equation*}
    \Ll^*_{\text{\rm stw}} = G^*_{\text{\rm stw}} \leq G^*_{\text{\rm exp}} = \Ll^*_{\text{\rm exp}}
\end{equation*}
Additionally, if the Slater's condition holds on $\Pi$ for each $s$, then the strong duality holds \citep{bertsekas1997nonlinear}. The total rewards lower bound can be further obtained:
\begin{equation*}
	J^*_{\text{\rm stw}} = - \Ll^*_{\text{\rm stw}} \geq -\Ll^*_{\text{\rm exp}} = J^*_{\text{\rm exp}}
\end{equation*} $\hfill\blacksquare$

\section{Baseline Algorithms and Implementations}
\label{sec:implementations}
\subsection{Implementations}

Implementation of FAC are based on the Parallel Asynchronous Buffer-Actor-Learner (PABAL) architecture proposed by \citet{guan2021mixed}. All experiments are implemented on Intel Xeon Gold 6248 processors with 12 parallel actors, including 4 workers to sample, 4 buffers to store data and 4 learners to compute gradients.\footnote{Our open-source implementation of FAC can be found at \url{https://github.com/mahaitongdae/Feasible-Actor-Critic}. The original implementation of PABAL can be found at \url{https://github.com/idthanm/mpg}.  The baseline implementation is modified from \url{https://github.com/openai/safety-starter-agents} and \url{https://github.com/ikostrikov/pytorch-a2c-ppo-acktr-gail}.}
\subsection{Hyperparameters}
The hyperparameters of FAC and baseline algorithms are listed in Table \ref{table.hyper}.
\begin{table*}[htp]
	\vskip 0.15in
	\begin{center}
		\begin{tabular}{lc}
			\toprule
			Algorithm & Value \\
			\hline
			\emph{FAC} & \\
			\quad Optimizer &  Adam ($\beta_{1}=0.9, \beta_{2}=0.999$)\\
			\quad Approximation function  &Multi-layer Perceptron \\
			\quad Number of hidden layers & 2\\
			\quad Number of hidden units per layer & 256\\
			\quad Nonlinearity of hidden layer& ELU\\
			\quad Nonlinearity of output layer& linear\\
			\quad Actor learning rate & Linear annealing $3{\rm{e-}}5\rightarrow1{\rm{e-}}6 $\\
			\quad Critic learning rate & Linear annealing $8{\rm{e-}}5\rightarrow1{\rm{e-}}6 $\\
			\quad  Learning rate of multiplier net & Linear annealing $5{\rm{e-}}5\rightarrow5{\rm{e-}}6 $ \\
			\quad  Learning rate of $\alpha$ & Linear annealing $5{\rm{e-}}5\rightarrow1{\rm{e-}}6 $ \\
			\quad Reward discount factor ($\gamma$) & 0.99\\
			\quad Cost discount factor ($\gamma_c$) & 0.99\\
			\quad Policy update interval ($m_\pi$) (speed limit tasks)& 2\\
			\quad Policy update interval ($m_\pi$) (safe exploration tasks)& 4\\
			\quad  Multiplier ascent interval ($m_\lambda$) (speed limit tasks)& 6\\
			\quad  Multiplier ascent interval ($m_\lambda$) (safe exploration tasks)& 12\\
			\quad Target smoothing coefficient ($\tau$) & 0.005\\
			\quad Max episode length ($N$) & 1000\\
			\quad  Expected entropy ($\overline{\mathcal{H}}$) &  $\overline{\mathcal{H}}=-$Action Dimentions \\
			\quad  Replay buffer size & $5\times10^5$\\
			\quad  Reward scale factor (speed limit tasks)& 0.2 \\
			\quad  Reward scale factor (safe exploration tasks)& 1\\
			\quad  Replay batch size & 256\\\midrule
			\emph{CPO, TRPO-Lagrangian} &\\ 
			\quad Max KL divergence&  $0.1$\\
			\quad Damping coefficient&  $0.1$\\
			\quad Backtrack coefficient&  $0.8$\\
			\quad Backtrack iterations&  $10$\\
			\quad Iteration for training values&  $80$\\
			\quad Init $\lambda$ &  $0.268 (softplus(0))$\\
			\quad GAE parameters  &  $0.95$\\
			\quad Batch size &  $2048$\\
			\quad Max conjugate gradient iterations $ 10$ \\
			
			\hline
			\emph{PPO-Lagrangian} &\\ 
			\quad Clip ratio &  $0.2$\\
			\quad KL margin &  $1.2$\\
			\quad Mini Bactch Size & $64$\\
			\bottomrule
		\end{tabular}
	\end{center}
	\vskip -0.1in
	\caption{Detailed hyperparameters.}
	\label{table.hyper}
\end{table*}

\section{Additional Experiment Details}
\label{sec:appendixexp}
\subsection{Details about Emergency Braking Task}
\label{sec:embrake}
The emergency braking tasks includes a static obstacles, and a vehicle. The reward is to decelerate as less as possible:
\begin{equation}
    r_t = {a^{\text{del}}_t}^2
\end{equation}
As we want the safety constraint $d_t\leq 0$ holds for each time $t$, we use a finite horizon design for finite constraint in one initial state:
$J(\pi)=\sum_{t=0}^N r_t$
the safety constriant:
$C_i(\pi) = d_i \geq 0$, where $i\in{1,2,...10}$ represents the prediction steps. The initial state distribution is a uniform distribution on the theoretical feasible region $v^2\leq 2 |a^\text{del}_{\text{max}}|d$. We use the vanilla PG method \citep{sutton1999policy} added with a expectation-based Lagrangian to solve this problems. The resulting region in Figure 1(ii) is estimated by sampling approximation. We sample the initial state every 0.1 m or m/s in the state space, which adds up to $100*100$ samples.  Then for each initial states, we use the trained policy to drive the car, and see if it will collide with the obstacles. 
\subsection{Safety Critic Constraint Transformation}
\label{sec:valuecstr}
A real-world safety-oriented constraint is usually the limited frequency of dangerous action frequency in the continuous tasks, or the limited number of cost signal $1$ in a $N$-step episode \citep{ray2019benchmarking}. We use the binary formulation of cost function $c(s_t,a_t)$, where $c(s_t,a_t)=1$ if the action is dangerous otherwise $c(s_t,a_t)=0$. We further assume the cost signal $c_t=1$ occurs uniformly in the sampling. We use the $d_{\text{rate}}$ to denote the frequency threshold (the limited number in a $N$-step episode is $d_{\text{rate}}N$ correspondingly), then a real-world safety-oriented dangerous action frequency constraint can be transferred to a discounted value safety constraint with discounted factor $\gamma_C$:
\begin{equation}
	d = \frac{d_{\text{rate}}N(1-\gamma_C^N)}{N(1-\gamma_C)} = \frac{d_{\text{rate}}(1-\gamma_C^N)}{(1-\gamma_C)} 
\end{equation}
In practice, $\gamma_C<1$ and $N$ is a large integer in the episodic task case or infinite in the continuous task case, so $\gamma_C^N$ can be neglected, and the relation can be simplified as $d\approx d_{\text{rate}}/(1-\gamma)$. In our safe exploration tasks, $d=0.1/(1-0.99)=10$.
\subsection{Boundary of Infinite Multipliers}
\label{sec:boundary}
We use the following equation of estimate the boundary between finite and infinite multipliers. For the optimal Lagrange multiplier at state $s$ of \eqref{eq:asl}, the following equation holds:
\begin{equation}
    \nabla_\theta v^\pi(s) + \sum_s \lambda(s)\nabla_\theta v^\pi_C(s) = 0
\end{equation}
Therefore, the scale of $\nabla v^\pi(s)/\nabla v^\pi_C(s)$ can be  a heuristic threshold for the multipliers. A simplest choice is the ratio of average gradient norm of objective and constraint function, as the threshold, which is what we use in the safe exploration tasks. More dedicated design will be considered in the future works.
\vskip 0.2in
\bibliography{bib/nips}

\end{document}